\documentclass[letterpaper, 10 pt, conference]{ieeeconf}

\makeatletter

\let\proof\@undefined
\let\endproof\@undefined
\makeatother

\IEEEoverridecommandlockouts                              
\usepackage[T1]{fontenc}
\usepackage{graphicx}
\graphicspath{ {./images/} }

\makeatletter
\newcommand{\eqnum}{\refstepcounter{equation}\textup{\tagform@{\theequation}}}
\makeatother
\usepackage{amssymb}
\usepackage{amsthm}
\usepackage{mathtools}
\usepackage{algorithm}
\usepackage[noend]{algorithmic}
\usepackage{comment}
\usepackage[hidelinks]{hyperref}
\usepackage{booktabs}

\usepackage[sort,compress]{cite}
\usepackage{array, makecell} 
\usepackage[font=small]{caption}

\newcommand{\RN}[1]{%
	\textup{\uppercase\expandafter{\romannumeral#1}}%
}
\usepackage{color}

\definecolor{mypink}{rgb}{0.858, 0.188, 0.478}

\definecolor{mygray}{rgb}{0.5, 0.5, 0.5}

\newcommand{\real}{\mathbb{R}}
\newcommand{\cl}[1]{{\mathcal{#1}}}

\theoremstyle{definition}
\newtheorem{definition}{Definition}
\newtheorem{problem}{Problem}
\newtheorem{lemma}[]{Lemma}
\newtheorem{proposition}[]{Proposition}

\def\vd{3pt}

\begin{document}
	
	\title{Scheduling Operator Assistance for Shared Autonomy in Multi-Robot Teams}
	\author{Yifan Cai, Abhinav Dahiya, Nils Wilde, Stephen L. Smith
		\thanks{This research is supported in part by the Natural Sciences and Engineering Research Council of Canada (NSERC) and in part by the Innovation for Defence Excellence and Security (IDEaS) Program of the Canadian Department of National Defence through grant CFPMN2-037.} 
		\thanks{Y.\ Cai, A.\ Dahiya, and S.\ L.\ Smith are with the Department of Electrical and Computer Engineering, University of Waterloo, Waterloo, ON N2L 3G1, Canada \{{\tt\small yifan.cai, abhinav.dahiya, stephen.smith\}@uwaterloo.ca}.  N.\ Wilde is with the Cognitive Robotics Department, Delft University of Technology, Netherlands ({\tt\small N.Wilde@tudelft.nl}).}
	}
	
	\maketitle

	\begin{abstract}
		
		In this paper, we consider the problem of allocating human operator assistance in a system with multiple autonomous robots.  Each robot is required to complete independent missions, each defined as a sequence of tasks.  While executing a task, a robot can either operate autonomously or be teleoperated by the human operator to complete the task at a faster rate.  We show that the problem of creating a teleoperation schedule that minimizes makespan of the system is NP-Hard.  We formulate our problem as a Mixed Integer Linear Program, which can be used to optimally solve small to moderate sized problem instances.  We also develop an anytime algorithm that makes use of the problem structure to provide a fast and high-quality solution of the operator scheduling problem, even for larger problem instances.  Our key insight is to identify \textit{blocking tasks} in greedily-created schedules and iteratively remove those blocks to improve the quality of the solution.  Through numerical simulations, we demonstrate the benefits of the proposed algorithm as an efficient and scalable approach that outperforms other greedy methods.
	\end{abstract}
	\section{Introduction}
	Autonomous mobile robot teams have been widely used in manufacturing and related sectors resulting in improved productivity and reduced risk to human workers. 
	Such robot teams are able to function autonomously on their own, while also bearing the capability of making use of human assistance to further improve their performance \cite{rosenfeld2017intelligent, khasawneh2019human, zheng2013supervisory}.  As it is challenging for human operators to supervise and assist a large number of robots on their own \cite{chen2012supervisory, chien2013imperfect}, a number of studies in the literature propose effective decision support systems (DSS) to aid the human operator(s) in providing assistance \cite{dahiya2021scalable, swamy2020scaled, rosenfeld2017intelligent}.

	
	In this paper, we present such a DSS for a multi-robot system comprising a fleet of autonomous robots with a human operator available to teleoperate the robots to speed up their missions, given their availability.  Figure~\ref{fig:main} presents an overview of the problem setup, showing $K$ robots navigating in a city-block-like environment and going through a series of tasks.  A task in this example may refer to navigating through the robot route, crossing a road, going through a crowded area, and etc. 
	Each task is characterized by different completion times, depending on whether the task is executed autonomously or under teleoperation. There is a human operator available, who can assist/teleoperate at most one robot at a time. All robots are capable of completing their respective tasks on their own, but can be assisted by a human operator to speed up the task completion. The DSS provides the operator with a teleoperation schedule that specifies which task of a robot should be executed via teleoperation, and in what order. If a robot task is scheduled for teleoperation, then the robot and operator must \emph{wait} for each other to be available before starting this task. Thus, a schedule specifies the teleoperation actions for the operator and the wait actions for all robots and the operator.
	
	The problem objective is to find a teleoperation schedule for both the human operator and robots that minimizes the time taken until all robot missions are complete.
	%
	
	\begin{figure}[t]
		\centering
		\includegraphics[width=0.9\columnwidth]{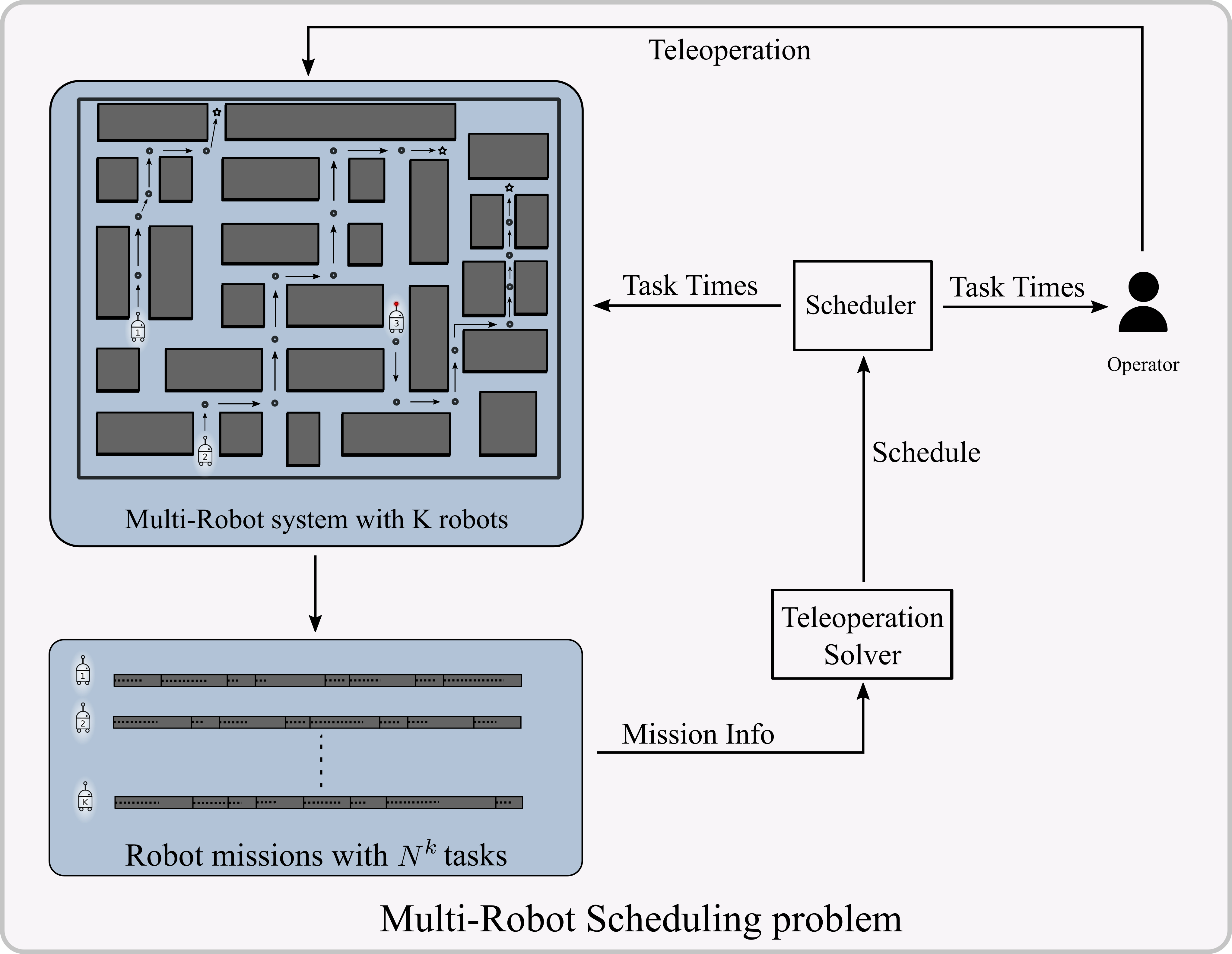}
		\caption{Information flow in the multi-robot teleoperation scheduling system with $K$ robots. A robot $k$ are assigned with an independent series of  tasks. Given mission information, solver computes the schedule, which is then converted to information about task timing for both operator and robots.  The operator assists on tasks assigned by the schedule.}
		\label{fig:main}
	\end{figure}
	Our work makes the following contributions:
	
	1) We show that the operator scheduling problem for multiple robots is NP-Hard using a reduction from a variant of the Satisfiability problem called \textit{2p1n-3SAT} problem. 
	
	2) We formulate a Mixed Integer Linear Program (MILP) that can be used to generate optimal schedules for the given problem.  We also present an extension to a multiple-operator version of the problem.
	
	3) We develop an anytime algorithm that iteratively generates teleoperation schedules for the given problem.  The algorithm is capable of solving much larger instances of the given problem than the MILP formulation.
	
	4) We evaluate our proposed algorithm in numerical simulations.  The results show that our method provides an efficient and scalable solution compared to other approaches.
	
	\subsection{Background and Related Works}
	\label{sec:background}
	Human-multi-robot teams have found their application in search-and-rescue \cite{ren2017cooperation}, smart factory operation \cite{huang2019multi}, home care for seniors \cite{benavidez2015design}, and package delivery \cite{dahiya2021scalable}.  However, such a team composition also brings the risk of increasing operator workload and decreasing in their situational awareness  \cite{wong2017workload, riley2005situation}.  In \cite{chien2012scheduling}, the authors show that scheduling the operator's attention can improve the efficiency of control over multi-robot system. Therefore, such systems can benefit from having a DSS that decides how to distribute human assistance among different robots or autonomous systems \cite{chen2014human, rosenfeld2017intelligent}. 
	
	The problem of scheduling human assistance among multiple robots has similarities with disciplines of multi-robot supervision, queuing theory, and task scheduling and sequencing.
	All these studies propose some forms of DSS, where an advising agent guides the human operator(s) on a robot (or task) which they should assist, with specified time.  This advice can take form of an online allocation, like in \cite{dahiya2021scalable}, or a pre-determined offline schedule, like in \cite{hari2020approximation}.
	In human-supervised multi-robot systems, frameworks such as sliding autonomy that considers factors like coordination and situational awareness are shown to improve understanding of such systems~\cite{music2017control, dias2008sliding}.  Research on effective interaction interfaces also aims to facilitate human supervision of robot teams \cite{szafir2017designing, kirchner2016intelligent}.  Our work is concerned with providing instructions to human operator on how to allocate their attention among different robots.  
	%
	%
	In the queuing discipline, efficient techniques have been developed to enable a human to service a queue of tasks \cite{gupta2019optimal}.  However, the model that we study is different from a queuing model as it is possible for the robots to complete their tasks without the help of operators, and there is no pre-defined order in which tasks (of different robots) are required to be processed.  
	
	
	Related studies in scheduling literature present methods to schedule processing of different tasks to minimize performance metrics like makespan, idle time etc.  A common way of solving the scheduling problem is through the MILP formulation, which can be used to obtain optimal solutions for scheduling problems.  In the literature, we also find scalable methods to approximately solve a MILP for large instances which may take MILP hours to find the minima.  For example, the study presented in \cite{roslof2002solving} makes use of a heuristic procedure for a single machine job scheduling.  However, in our system not all tasks are required to be scheduled and tasks from different robots are not required to be in any particular order. Methods like rolling-horizon splits problems into smaller pieces based on time and pursue the local optimal \cite{bischi2019rolling}.  In contrast to the problem considered in \cite{bischi2019rolling}, our problem is highly-coupled over time, and thus there aren't natural breakpoints in time to decompose the problem. 
	
	
	The most related works to our problem are presented in \cite{zanlongo2021scheduling} and \cite{hari2020approximation}.  These studies propose solutions to scheduling of operators, and robot planning for multi-robot system having critical configurations where operator attention/input is required to proceed.  
	While sharing a similar goal with these studies (minimizing mission time), our system lacks the presence of any such critical configurations or states, and every task can be completed both autonomously and under teleoperation.  
	
	
	
	\section{Multi-robot Teleoperation Scheduling}

	
	We consider a system consisting of a human operator supervising a fleet of $K$ autonomous robots. Each robot $k\in {\cal K} \coloneqq \{1,\ldots, K\}$ is assigned a mission $p^k \in \cl{P} := \{p^1,\ldots, p^K\}$, which is a pre-defined sequence of tasks.  To complete its mission $p^k$ the robot $k$ is required to complete $N_k$ tasks. 
	The $j^{th}$ task of robot $k$ is denoted as $e^k_{j}$.  For each task, a robot can either operate autonomously or be teleoperated by the human operator.  Executing a task $e^k_{j}$ takes time $\alpha^k_{j}$ if the robot operates autonomously and time $\beta^k_{j} (\leq \alpha^k_{j})$ if it is teleoperated\footnote{In this paper we consider $\beta^k_{j} \leq \alpha^k_{j}$, i.e., teleoperation is at least as fast as autonomous operation.  However, even for cases when this condition does not hold, the analysis and algorithms presented in this paper apply without any changes, as the tasks where autonomous operation is faster than teleoperation are not considered for scheduling.}.
	
	There is a DSS that provides a teleoperation schedule for the operator.  A complete teleoperation schedule contains the information of when to start each task of every robot and which of the tasks are teleoperated.  This information also tell us if a robot or an operator needs to wait before starting a task. However, since the completion times for each task are known, this teleoperation schedule can be presented in a more compact form as only a sequence of teleoperated tasks. The timing information can be computed in polynomial time from this sequence using the time $\alpha^k_{j}$ and $\beta^k_{j}$.  
	
	For our problem, we consider a schedule $\cl{S}$ as a sequence of tasks $\langle s_1, \dots, s_n \rangle$ where each $s_i$ corresponds to some task $e^k_j$ for $k\in\{1,\ldots,K\}, j\in\{1,\ldots,N^k\}$ that is required to be teleoperated.  Once the mission starts, the human operator teleoperates the specific tasks in the order provided by the schedule $\cl{S}$, i.e., task $s_1$ followed by $s_2$ and so on.  If at the end of some task $s_i = e^k_j$, the robot $k$ is not yet ready for the required task (executing its previous tasks), the operator waits for the robot to arrive at the start of $e^k_j$.  Likewise, if the robot is ready for the task $s_i$, but the operator is still working on a previous task $s_{i'}$ where $i' < i$, then the robot waits for the operator.  
	
	The mission ends when all robots complete their respective sequence of tasks.  A common metric of measuring performance of such systems is the time elapsed until all robot missions are complete, called the \textit{makespan} \cite{mau2006scheduling}, denoted as $\mu(\cl{S}) \in \real_{>0}$.
	
	\subsection{Problem Statement} 
	We impose the following assumptions on the problem:
	\begin{enumerate}
		\item [\textbf{(A1)}] The operator can teleoperate at most one robot at a time.  
		\item [\textbf{(A2)}] A task's mode of operation cannot change once the task is started, i.e., an operator must teleoperate a robot throughout a task, and they cannot join a task which already started autonomously.
		\item [\textbf{(A3)}] All robots may start with the first task in their respective missions at or after the time $t=0$.
	\end{enumerate}
	
	The objective is to solve the following optimization.
	\begin{problem}\label{prob:main}
		Given the set $\cl{K}$ of robots, the missions  $\{p^1,\ldots, p^K\}$ for each robot, and the autonomous and teleoperation completion times $\alpha^k_j$ and $\beta^k_j$ for each task, find a schedule $\cl{S}$ that minimizes the makespan $\mu(\cl{S})$.
	\end{problem}
	To begin, we establish that this problem is NP-Hard.
	
	\subsection{Hardness Proof: Reduction from 2p1n-3SAT}
	To prove Problem~\ref{prob:main} is NP-hard, we introduce an NP-complete variant of Satisfiability called \textit{2p1n-3SAT}\cite{yoshinaka2005higher}.  
	In \textit{3SAT} problems, we are given a Boolean formula as a conjunction of several clauses where each clause is a disjunction of exactly 3 literals.  A literal is either a variable or its negation.  In \textit{2p1n-3SAT}, each variable shows up exactly twice, and its negation shows up exactly once \cite{yoshinaka2005higher}.  
	The \textit{2p1n-3SAT} problem and the decision version of the scheduling Problem~\ref{prob:main} are as follows: 
	
	\begin{problem} [\textit{2p1n-3SAT}] 
		Given a Boolean expression in the \textit{2p1n} format with $K$ clauses with $v$ variables, does there exist an assignment to the variables that makes the formula evaluate to true? 
		\label{prob:2p1n}
	\end{problem}
	
	\begin{problem}  [\textit{Min-Makespan}] Given $K$ robots, the missions $\{p^1,\ldots, p^K\}$ for each robot, and the autonomous and teleoperation task completion times $\alpha^k$ and $\beta^k$ for every task of each robot\footnote{Since task completions are same for all tasks, we omit subscript $j$ from expressions of $\alpha$ and $\beta$ for ease of notation.}, and a target time $\overline{\mu}\in\real_{\geq0}$, does there exist a schedule $\cl{S}$ that results in a makespan $\mu(\cl{S}) \leq \overline{\mu}$? 
		\label{prob:schedule-D}
	\end{problem}
	
	\begin{proposition} \label{prop:NP}
		Problem~\ref{prob:schedule-D} (\textit{Min-Makespan}) is NP-Complete.
	\end{proposition}
	\begin{proof}
		We begin the proof by proposing a reduction from Problem~\ref{prob:2p1n} to Problem~\ref{prob:schedule-D}:
		
			\emph{Reduction: }
			To reduce an instance $\phi$ of Problem~\ref{prob:2p1n} into an instance $\psi$ of Problem~\ref{prob:schedule-D}, we replace the literals in the Boolean formula with tasks in robots missions.  Each clause in the formula corresponds to a robot in the scheduling problem.  Fixing the order of variables arbitrarily, we create robot a mission from literals in a clause using the following four rules:
			
			\begin{enumerate}
				\item If a literal is the first positive appearance of a variable in the Boolean formula $\phi$, we add two consecutive tasks in robot's mission.  The first task has autonomous completion time $\alpha = z$ and teleoperation time $\beta = z-\delta z$, for some $z,\delta z\in\mathbb{Z}_{>0}$ and $0 < \delta z \ll z$ (e.g., $z=100, \delta z=1$).  For the second task, $\alpha = \beta = z$.
				
				\item If a literal is the second positive appearance of a variable in $\phi$, we again add the same two tasks as the first case, but in reverse order. 
				
				\item If a literal is the negative appearance of a variable, we add a single task in robot's mission with $\alpha = 2z$ and $\beta = 2z - \delta z$.
				
				\item For each variable that is not present in the clause (when there are more than three variables), we add a single task in robot's mission with $\alpha = \beta = 2z$.
				
			\end{enumerate}
			
			
			With this reduction, an instance of Problem~\ref{prob:2p1n} with $v$ variables is converted to an instance of Problem~\ref{prob:schedule-D} with $\overline{\mu} = 2zv$.
			An example reduction is illustrated in Fig.~\ref{fig:trans}.
			
			
			\begin{figure}[thpb]
				\centering
				\includegraphics[width=\columnwidth]{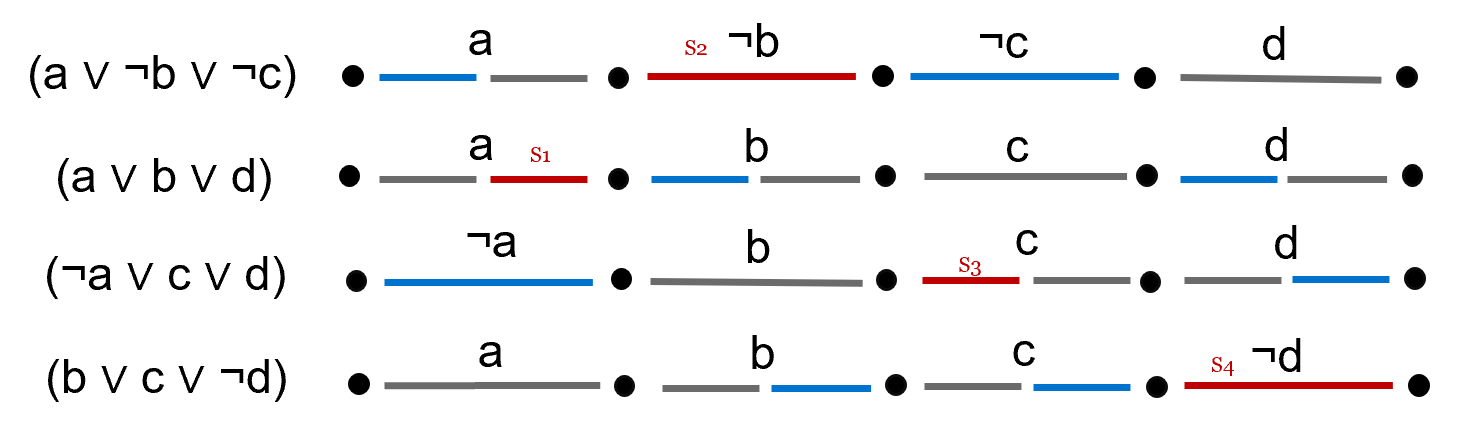}
				\caption{Converting \textit{2p1n-3SAT} formula $\phi$ with four clauses to missions for four robots.  Each mission is shown as a sequence of differently-colored tasks.  Blue and red tasks have a difference of $\delta z$ between their $\alpha$ and $\beta$, and grey tasks have no difference between their $\alpha$ and $\beta$.}
				\label{fig:trans}
			\end{figure}
			
		
		Under the reduction given in Proposition~\ref{prop:NP}, the generation of $\psi$ from $\phi$ takes polynomial time as one needs to parse each literal in every clause only once, and generate tasks in robots missions (a constant time operation) for each of those literals. 
		
		Trivially, \textit{Min-Makespan} is in NP. Given a certificate $\langle \cl{S}, \psi \rangle$, it takes polynomial time to verify the certificate, by computing the makespan of $\psi$.
		
		Therefore, the NP-Completeness of Problem~\ref{prob:schedule-D} follows directly from Lemmas~\ref{lem:fwd_implication} and~\ref{lem:rev_implication}.
	\end{proof}

	\begin{lemma} \label{lem:fwd_implication}
		Under the reduction given in Proposition~\ref{prop:NP}, if $\phi$ is a true instance of Problem~\ref{prob:2p1n}, then $\psi$ is a true instance of Problem~\ref{prob:schedule-D}.  
	\end{lemma}
	\begin{proof}
		For the purpose of the proof, we call the tasks in $\psi$ where $\alpha > \beta$ as \textit{effective} tasks.  
		In order to have $\phi$ to be true, each clause must have at least one true literal by definition.  The reduction is designed in a way that effective tasks corresponding to positive variables in different clauses do not overlap with each other.  Therefore, if $\phi$ is a true instance of Problem~\ref{prob:2p1n}, then it is possible to teleoperate the effective tasks corresponding to at least one true literal in each clause.
		
		This will result in a reduction of at least $\delta z$ in the travelling time of each robot, so the makespan is less than $2zv$.  For example, Fig.~\ref{fig:trans} shows a true instance of the problem, and the schedule marked as red-colored tasks $\langle s_1, s_2, s_3, s_4 \rangle$ results in a makespan $\mu(\cl{S}) = 2zv - \delta z$. 
		%
	\end{proof}
	
	\begin{lemma} \label{lem:rev_implication}
		Under the reduction given in proposition~\ref{prop:NP}, if $\psi$ is a true instance of Problem~\ref{prob:schedule-D}, then $\phi$ is a true instance of Problem~\ref{prob:2p1n}.  
	\end{lemma}
	\begin{proof} 
		Since every robot in $\psi$ has the same autonomous task completion time, each of them must have at least one effective task teleoperated to have a makespan less than $2zv$.  The tasks are arranged in a way that a variable and its negation's effective task cannot both present in a satisfying schedule (due to overlap in their times).  
		
		Therefore, we can choose any one teleoperated task from each robot mission, and set its corresponding literal to be true.  This will result in an assignment for which $\phi$ is true.
	\end{proof}

	Since the decision problem \textit{Min-Makespan} is in NP-Complete, the problem of finding the optimal teleoperation schedule in Problem~\ref{prob:main} is NP-Hard.  
	Since all constraints in our problem are linear time constraints, we formulate our problem as a mixed integer linear program (MILP), which has been used to formulate a wide range of NP-Hard problems.  In the next section, we provide details on how Problem~\ref{prob:main} can be encoded as a MILP.

	\section{MILP Formulation}
	\label{sec:MILP}
		
		
		
	
	In the MILP formulation, our objective is to find a schedule $\cl{S}$ that minimizes team makespan $\mu(\cl{S})$, subject to conditions on system dynamics and task ordering.
	%
	We begin by introducing three variables for each task: (1) $x^k_j$, a binary teleoperation variable for task $e^k_j$, (2) $\tau^k_j$, the scheduled start time for $e^k_j$, and (3) $\varepsilon^k_j$, the finish time for $e^k_j$, which can be expressed as a sum of the $\tau^k_j$ and the task completion time under the schedule, i.e.,
	$$\varepsilon^k_j = \tau^k_j + (1-x^k_j)\,\alpha^k_j + x^k_j\, \beta^k_j.$$
	A MILP can then be formulated as follows:
	\begin{align}
		\text{Minimize:}\quad &  \bar{\mu} \nonumber\\
		\text{Subject to:}\quad
		& \bar{\mu} \geq \varepsilon^k_{N^k}\quad\,\,\,\forall~k\in\cl{K}, \label{c1}\\[\vd]
		& \tau^k_1 \geq 0, \quad\,\,\,\,\,\forall~k\in\cl{K},  \label{c2}\\[\vd]     
		& \tau^k_{j} \geq \varepsilon^k_{j-1},\;\,\forall~k\in\cl{K},\, j \in \{2,\ldots,N^k\},\label{c3}\\[\vd]
		& x^k_j + x^l_i = 2 \implies
			\tau^k_j \geq \varepsilon^l_i \,\text{ or }\,
			\tau^l_i \geq \varepsilon^k_j,
		\nonumber\\ 
		&\quad\qquad\qquad\, \forall~k,l \in\cl{K};\, k\neq l,  \nonumber\\
		&\quad\qquad\qquad\, \forall~j \in \{1,\ldots, N^k\},\nonumber\\ 
		&\quad\qquad\qquad\, \forall~i \in \{1,\ldots, N^l\},\label{c4}\\[\vd]    
		%
		%
		\quad & x^k_j\in \{0,1\},\forall~k\in\cl{K},\, j\in \{1,\ldots,N^k\}\label{c5}.
	\end{align}
	
	Constraint~\eqref{c1} restricts the time needed for every robot to complete its mission to be not more than the objective $\bar{\mu}$.  Constraint~\eqref{c2} sets the earliest start time for the robots.  
	%
	Constraint~\eqref{c3} ensures that the $j^{th}$ task of a robot mission can only starts after the $j-1^{th}$ task is completed.
	%
	Constraint~\eqref{c5} restricts the variables $x^k_j$ to be a binary variable. 
	Constraint~\eqref{c4} specifies no two tasks can be teleoperated with an overlapping time interval.
	The exclusive disjunction (XOR) of two conditions is required due to the undetermined order of teleoperation of the two tasks.  Note that constraint~\eqref{c4} above is presented as an implication and is not written as a linear constraint.  However, it can be converted to a set of linear constraints (for example, by using the Big-M method), which are supported directly by many mixed integer linear program solvers \cite{brown2007formulating}.  
	
	\textbf{Note:} To implement constraint~\eqref{c4}, we can limit the ranges to $k\in \{1,\ldots,K-1\},\, l \in\{k,\ldots, K\}$, which eliminates the repetitions in constraint checking, thus more efficient.

	\subsection{Extension to Multiple Operators}
	It is worth noting that we can directly extend the MILP to handle the multi-operator-multi-robot setting.  In this case we have a set of $M$ operators $\cl{M} := \{1, \ldots, M\}$, and use binary variable $x^k_{jm}$ to indicate whether $e^k_j$ is teleoperated by operator $m\in\cl{M}$. 
	%
	Whether a task is teleoperated or not is now indicated by $\sum_{m\in\cl{M}} x^k_{jm}$, instead of $x^k_j$. Consequently, changes are made in expressions for $\varepsilon^k_j$ and Constraint~\eqref{c4}.  Constraint~\eqref{c5} is repeated for all $x^k_{jm}$.
	
	
	In addition, we need a constraint to bound $\sum_{m\in\cl{M}} x^k_{jm},$ since each task can be assigned to at most one operator:
	\begin{equation}
		\label{eq:one_operator}
		\sum_{m\in\cl{M}} x^k_{jm} \leq 1, \,\forall\, k \in \{1,\ldots,K\},\, j \in \{1,\ldots, N^k\}.
	\end{equation} 
	
	\subsection{Solving the MILP}
	A globally optimal solution to a MILP can be found using solvers like Gurobi or CPLEX.  However, as mentioned earlier, while such solvers are effective for small problem instance (i.e. 2 robots with 8 tasks each, such an instance takes about 8.3 sec for MILP), they do not scale to large instances (i.e. 3 robots with 15 tasks each, such an instance takes about 296 sec for MILP) involving many robots, each with ten or more tasks in its mission.  In the next section, we present an efficient algorithm that makes use of the problem structure to provide a fast and high-quality solution of Problem.~\ref{prob:main}.

	\section{Iterative Greedy}
		
		
	In this section, we present a greedy algorithm called $\mathtt{Iterative\, Greedy}$.  The algorithm begins by greedily creating a schedule to improve the team's makespan, until no further improvements can be made by adding tasks of a makespan robot to the schedule.  Our key insight here is to then identify \textit{blocking tasks} in such greedily-created schedules and iteratively remove those blockages to improve the solution.  The algorithm cycles between two routines: Greedy Insertion and Block Removal.
	
	\subsection{Greedy Insertion}
	This routine creates a teleoperation schedule by greedily selecting tasks from the mission of a robot whose total time currently equals the makespan (called a makespan robot).  
	\begin{definition}  [Greedy Insertion] 
		For a given schedule $\cl{S}$, let robot $k$ be a robot achieving the makespan (i.e., last task's finish time $\varepsilon^k_{N^k} = \mu(\cl{S})$).  We call the addition of a task $e^k_i$ to schedule $\cl{S}$ a Greedy Insertion if the addition of $e^k_i$ directly reduces $\varepsilon^k_{N^k}$, without increasing the team makespan.
		\label{def:direct}
	\end{definition}
	
	Pseudo-code for the $\mathtt{Greedy\, Insertion}$ algorithm is presented in Algorithm \ref{alg:direct}.  
	In the algorithm, given a schedule $\cl{S}$, we first identify the set of all makespan robots, denoted as $\overline{\cl{K}}$.  We then determine the \textit{best} task ${e}^k$, defined as the task that reduces $\varepsilon^k_{N^k}$, the mission time of any robot $k\in\overline{\cl{K}}$ by the most, while not increasing the makespan $\mu(\cl{S})$.  This task is then inserted in the schedule.  
	
	
	\begin{algorithm}
		\caption{$\mathtt{Greedy\, Insertion}$}
		\begin{algorithmic}[1]
			\renewcommand{\algorithmicrequire}{\textbf{Input:}}
			\renewcommand{\algorithmicensure}{\textbf{Output:}}
			\REQUIRE $\cl{P}$, $\cl{S}$
			\ENSURE  $\cl{S}'$
			\STATE Initialize $\Delta {\varepsilon}^* = 0$, $\cl{S}' = \cl{S}$ \\
			\STATE Calculate mission time $\varepsilon^k_{N^k}$ for $k\in\{1,\ldots,K\}$ given $\cl{P}$, $\cl{S}$
			\STATE $\overline{\cl{K}} \gets \arg\max_k \{\varepsilon^1_{N^1},\ldots,\varepsilon^K_{N^K}\}$
			\FOR{$k \in \overline{\cl{K}}$}
			\STATE Find the \texttt{Best} task ${e}^{k}$, with time reduction $\Delta {\varepsilon}^k_{N^k}$ and corresponding schedule $\cl{S}^k$ given $\cl{P}$, $\cl{S}$
			\IF{$\Delta {\varepsilon}^k > \Delta {\varepsilon}^*$}
			\STATE $\Delta {\varepsilon}^* \gets \Delta {\varepsilon}^k_{N^k}$; 
			$\cl{S}' \gets \cl{S}^k$
			\ENDIF
			\ENDFOR
			\RETURN $\cl{S}'$
		\end{algorithmic} 
		\label{alg:direct}
	\end{algorithm}
	
	\textbf{Note:} The \texttt{best} task in the Greedy Insertion algorithm is defined as the one which results in the most reduction in mission time of any makespan robot.
	
	An example is shown in Fig.~\ref{fig:direct} to illustrate its operation.  Robot 3 is the makespan robot, and has two tasks currently not in the schedule.  Select the one with more time reduction, and this addition to the schedule will reduce $\mu(\cl{S})$.
	\begin{figure}[thpb]
		\centering
		\includegraphics[width=0.8\columnwidth]{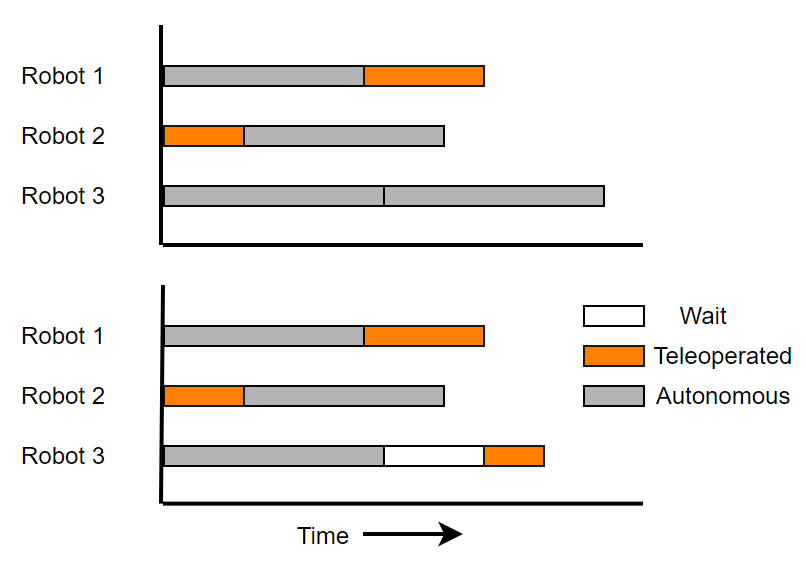}
		\caption{Example of Greedy Insertion. Robot 3 is the makespan robot, and by teleoperating its last task, we reduce its total mission time.}
		\vspace{-2ex}  
		\label{fig:direct}
	\end{figure}

	\subsection{Block Removal}
	A schedule created using $\mathtt{Greedy\, Insertion}$ gives a feasible (locally optimal) solution but such a schedule often results in considerable and frequent idle times for the operator.  However, even when the system's makespan cannot be improved further by adding more tasks of makespan robots to the schedule, it may be possible to improve the makespan by adding tasks from other robots.  This is the idea behind the Block Removal technique, which works by finding and eliminating \textit{blocking tasks} and reduces the team makespan by reducing waiting times in the schedule.  We begin to introduce the details of this technique with the following definitions.
	\begin{definition} [Idle Time]   
		For any two adjacent tasks $s_j$ and $s_{j+1}$ in schedule $\cl{S} = \langle s_1, \dots, s_n \rangle$, the idle time is defined as the time between task $s_j$ finish time and $s_{j+1}$ start time.  For $s_1$, if its start time $>0$, idle time is simply the start time of itself.
		\label{def:idletime}
	\end{definition}
	
	\begin{definition} [Blocking Task and Blocking Robot]  
		\label{def:block}
		A task $s_{j+1}$ in schedule $\cl{S}$ is called a blocking task if the idle time between $s_j$ and $s_{j+1}$ is greater than zero\footnote{Depending on the application, it may be useful to set a threshold $\epsilon \in \real_{>0}$ on the idle time between $s_j$ and $s_{j+1}$ to consider $s_{j+1}$ as a blocking task. For example, we can set $\epsilon = min\{\beta^k_j\}$. In this case, if there is an idle time less than the minimum teleoperation time, inserting any task here only delays the execution of later tasks in the schedule. Thus, such an idle time cannot help improve the makespan and we may skip it.}.  The robot to which task $s_{j+1}$ belongs to is called a blocking robot.
	\end{definition}
	A blocking task is called so because it prevents a task in the makespan robot's plan from getting teleoperated or being teleoperated at an earlier time.  Reducing the starting time of the blocking task indirectly results in a smaller makespan or allows for further makespan decrease in future iterations.
	
	With above, the Block Removal operation can be defined.
	\begin{definition} [Block Removal] 
		Given a schedule $\cl{S}$, let robot $k$ be a robot achieving the makespan (i.e., $\varepsilon^k_{N^k} = \mu$).  We call the addition of a task $e^{k'}_i$ from a non-makespan robot $k'$ to the schedule $\cl{S}$ a Block Removal if the addition of $e^{k'}_i$ reduces or allow futher reduction on $\varepsilon^k_{N^k}$, without increasing the team makespan.  Such addition results in removal of blockage (idle time removed or reduced) by the robot $k'$ in the schedule.
		\label{def:Indirect}
	\end{definition}
	
	Pseudo-code for the $\mathtt{Block\, Removal}$ algorithm is presented in Algorithm \ref{alg:block}.  In the algorithm, we start by finding the blocking task in the schedule with the largest start time. This is because for most of time, there is no idle time between blocking task with the latest start time and the makespan robot's last teleoperated edge, and blocking can be resolved efficiently. We then try to add a task from the blocking robot's mission to the schedule such that it reduces the start time of the blocking task. If such an addition is possible, we return the updated schedule, else we discard this task and move to the blocking task with next largest starting time, until we reach the beginning of the schedule.
	\begin{algorithm}
		\caption{$\mathtt{Block\, Removal}$}
		\begin{algorithmic}[1]
			\renewcommand{\algorithmicrequire}{\textbf{Input:}}
			\renewcommand{\algorithmicensure}{\textbf{Output:}}
			\REQUIRE $\cl{P}$, $\cl{S}$
			\ENSURE  $\cl{S}'$
			\STATE Initialize: $\cl{S}' = \cl{S}$
			\STATE $\overline{s} \gets$ blocking task with largest starting time \label{line:s}
			\STATE Find task $e'$ that reduces the start time of $\overline{s}$
			\IF{$e'$ exists}
			\RETURN Updated schedule $\cl{S}'$
			\ELSE 
			\STATE Go to line~\ref{line:s} and repeat for blocking task with next largest start time until no more blocking tasks are present 
			\ENDIF
			\RETURN $\cl{S}'$
		\end{algorithmic} 
		\label{alg:block}
	\end{algorithm}
	
	An example is shown in Fig.~\ref{fig:indirect} to illustrate this operation.  Given the schedule generated in Fig.~\ref{fig:direct}, further Greedy Insertion is not possible.  Adding task of $e^3_1$ to the schedule does not reduce makespan because the task final task of Robot3, $e^3_2$, will have to wait until the operator finishes the task $e^1_2$ (the blocking task).  Instead, if we add $e^1_1$ to $\cl{S}$, it reduces the makespan by reducing the start time of the blocking task $e^1_2$. 
	
	\begin{figure}[thpb]
		\centering
		\includegraphics[width=0.8\columnwidth]{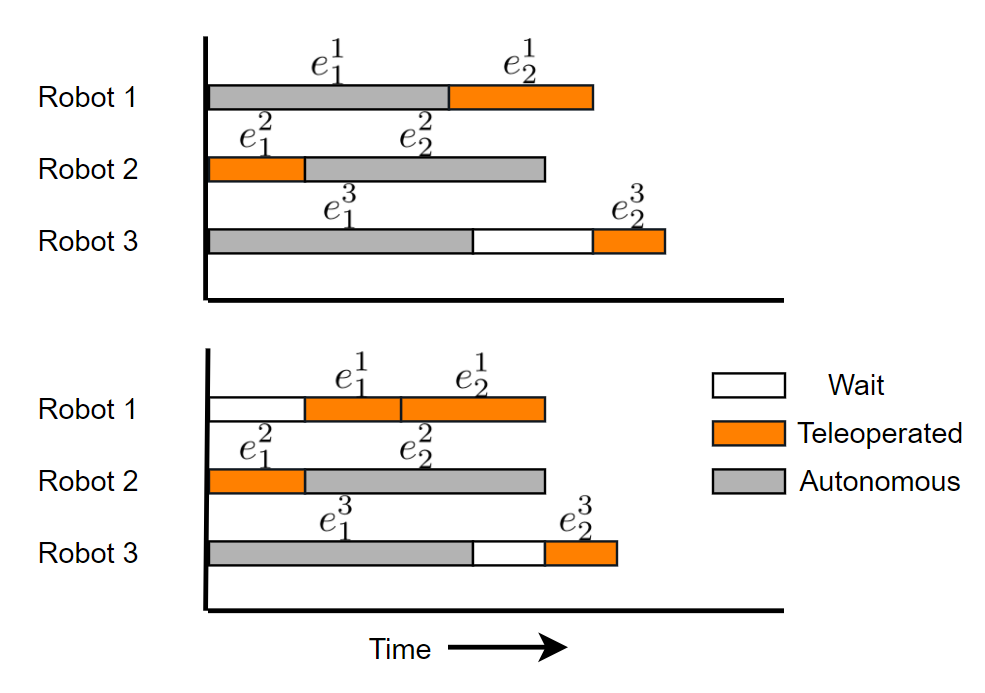}
		\caption{Example for Block Removal. Makespan Robot 3's mission time is reduced indirectly by teleoperating $e^1_1$.}
		\vspace{-2ex}
		\label{fig:indirect}
	\end{figure}

	\subsection{Iterative Greedy}
	Starting with an empty teleoperation schedule, the Iterative Greedy algorithm first generates an intermediate schedule using Alg.~\ref{alg:direct}.  Using this schedule, we try the Block Removal routine using Alg.~\ref{alg:block}.  The schedule is iteratively improved by applying Alg.~\ref{alg:direct} and Alg.~\ref{alg:block} one after the other, until both of these algorithms stop to make improvements in a given schedule $\cl{S}$, which is then selected as the final output.
	
	\begin{algorithm}
		\caption{$\mathtt{Iterative\, Greedy}$}
		\begin{algorithmic}[1]
			\renewcommand{\algorithmicrequire}{\textbf{Input:}}
			\renewcommand{\algorithmicensure}{\textbf{Output:}}
			\REQUIRE $\cl{P}$
			\ENSURE  $\cl{S}$
			\\\textit{Initialization} : $\cl{S} = [\textbf{ }]$, $done = 0$.
			\\\WHILE {\NOT $done$}
			\STATE $\cl{S}'$ $\gets$ Greedy Insertion($\cl{P}$, $\cl{S}$)
			\IF{$\cl{S}' = \cl{S}$}
			\STATE $\cl{S}' \gets$ Block Removal($\cl{P}$, $\cl{S}$)
			\IF{$\cl{S}' = \cl{S}$}
			\STATE $done = 1$
			\ENDIF
			\ENDIF
			\STATE $\cl{S} \gets \cl{S}'$
			\ENDWHILE
			\RETURN $\cl{S}$
			
		\end{algorithmic} 
		\label{alg:iterative}
	\end{algorithm}
	
	
	
	Runtime of $\mathtt{Iterative\, Greedy}$:  Letting $\bar N := \sum_{k=1}^K N^k$, each iteration of $\mathtt{Greedy\, Insertion}$ can be implemented to run in $O(\bar N)$ time.  Similarly, each iteration of $\mathtt{Block\, Removal}$ runs in $O(\bar N)$ time.  Since at most $\bar N$ tasks can be added to the schedule, the overall runtime of $\mathtt{Iterative\, Greedy}$ is bounded by $O(\bar N^2)$.

	\section{Evaluation} \label{sec:eval}
	In this section, we present performance results for a simulated multi-robot scheduling problem under the following methods (described in Section~\ref{sec:baselines}): 1) Optimal schedule (solution of the MILP formulation), 2) {Iterative Greedy}, 3) {Greedy Insertion}, 4) {Comparison Greedy}, and 5) {Na\"ive Greedy}.  
	The problem and the solution frameworks for all algorithms were implemented using Python.  The Gurobi Python API is used for the MILP solution.
	
	
	To generate an instance, for each task of each robot, two numbers are sampled from a uniform random distribution and are rounded to 2 decimal places. One is used as the task working time under teleoperation $\beta^k_j$, and the sum of two is used as the autonomous time $\alpha^k_j$:
	\begin{align}
		\beta^k_j \sim U[10,20], &\;\; \Delta \tau^k_j \sim U[0,10], \nonumber\\
		\alpha^j_j \gets \beta^k_j &+ \Delta \tau^k_j.
		\label{eq:sample_time}
	\end{align}

	\subsection{Baseline Algorithms}
	\label{sec:baselines}
	We consider the following baseline solution methods to assess the performance of the Iterative Greedy algorithm.
	
	\textbf{MILP Solution: }
	The MILP formulation in Section \ref{sec:MILP} is implemented and solved with Python Gurobi API.  Solving the formulation directly gives us $x^k_j$ and $\tau^k_j$ for each task.  
	
	\textbf{Na\"ive Greedy: }
	Under this algorithm, the operator is simply scheduled to teleoperate the next available task of the makespan robot.  If the makespan robot is still executing a task, the operator waits for the robot.
	
	\textbf{Comparison Greedy: }
	We have also developed the Comparison Greedy algorithm, which compares between alternatives given an intermediate schedule.  We compute the finish time of the last task in the current schedule, and determine the task $e^k_j$ that the makespan robot will be executing at that time.  We then pick the better of the two alternatives: 1) Adding $e^k_j$ to the schedule, and have the makespan robot wait for the operator at start of $e^k_j$, or 2) Adding $e^k_{j+1}$ to the schedule and have the operator wait for the makespan robot to complete~$e^k_j$.
	
	
	\textbf{Greedy Insertion: }
	To assess the improvement brought by the Block Removal step, we compare the schedule generated by only Greedy Insertion defined in Algorithm~\ref{alg:direct}.

	\subsection{Scalability Test} 
	We begin the evaluations by looking at the computation time of MILP and Iterative Greedy on different problem sizes (number of robots and tasks in their missions), as specified in Table~\ref{table_1}. The computation times shown in the table are the average of 100 instances for each case.  Along both dimensions of the problem size, the number of robots and number of tasks, the computation time of MILP increases at a very high rate.  Computation time of Iterative Greedy algorithm remains below 0.01 seconds for all test cases in Table~\ref{table_1}.  Even for larger instances, where MILP solution is unavailable, the computation time of Iterative Greedy algorithm grows at a much slower rate.  For example, for a problem instance with $4$ robots and $40$ tasks each, its average computation time is $5.22$ seconds. 
	For reference, the simulations were run on a laptop computer with 4 core, 2.1 GHz processor and 16 GB RAM.

	\begin{table}[h]
		\caption{CPU Time of MILP and Iterative Greedy (in seconds)}
		\label{table_1}
		\begin{center}
			\begin{tabular}{cccccc}
				\toprule
				& \makecell{$K = 2$\\ $N^k = 11$}& \makecell{$K = 3$\\ $N^k = 5$}& \makecell{$K = 3$\\ $N^k = 8$}& \makecell{$K = 3$\\ $N^k = 11$} &
				\makecell{$K = 4$\\ $N^k = 11$}\\
				\midrule \\[-1.5ex]
				MILP  & 0.30 &  0.4& 0.80 & 10.16 & 109.22\\[\vd]
				\makecell{Iterative\\ Greedy} & <\,0.01 & <\,0.01 & <\,0.01  & <\,0.01  & <\,0.01   \\
				\bottomrule
			\end{tabular}
		\end{center}
	\end{table}
	\vspace{-2ex}
	
	\subsection{Comparison with the Optimal Schedule}
	The Iterative Greedy algorithm is compared against the optimal schedule to validate its applicability for our problem.  The optimal schedule using MILP formulation cannot be computed for larger problem instances, due to its poor scalability, therefore this test is limited to small-sized problems. 
	The relative performance (ratio of the makespan under Iterative Greedy algorithm to the optimal schedule) is shown in Fig.~\ref{fig:compareOptimal}.
	For each size, $100$ instances were generated using the random instance generation mentioned earlier.
	
	We observe that the performance of the Iterative Greedy algorithm is comparable to that of optimal schedule.  As the number of robots increases, the distribution of relative performance slowly shifts away from 1.  However, the makespan under the Iterative Greedy algorithm is still within $5\%$ of the optimal schedule for over $90\%$ of the instances under all test cases.  For reference, the team makespan without teleoperation is, on average, $20.73\%$ more than the optimal for these test cases.
	
	\begin{figure}[]
		\centering
		\includegraphics[width=\columnwidth]{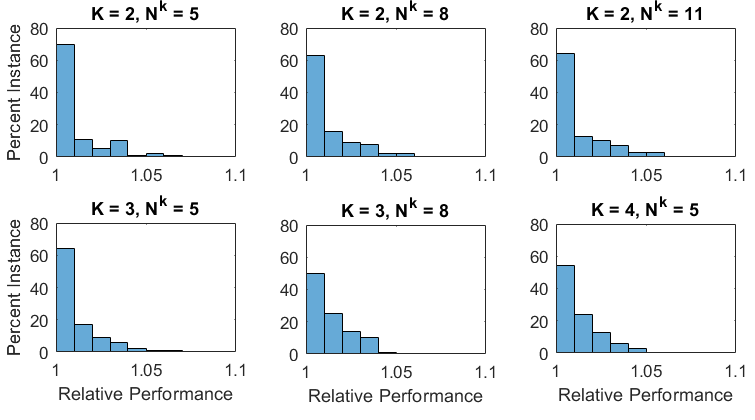}
		\caption {Relative performance of the Iterative Greedy methods compared to the optimal solution for number of robots $K\in\{2,3,4\}$, and number of tasks $N^k\in \{5, 8, 11\}$ for all robots.  Each plot shows the distribution of 100 instances based on their relative performance (ratio of makespan under Iterative Greedy method to the optimal makespan).}
		\vspace{-2ex}
		\label{fig:compareOptimal}
	\end{figure}
	
	
	\subsection{Comparison with other Greedy Algorithms} \label{sec:comparisonGreedy}
	Next, we compare the performance of the Iterative Greedy algorithm with the Greedy Insertion, Comparison Greedy and Na\"ive Greedy algorithms on larger problem instances.  Note that it is also possible to combine the Iterative Greedy algorithm with any of these greedy algorithms.  We include performance results from such combinations to demonstrate its effects on greedily-generated schedules.
	For the comparison, under each test condition (given number of robots and tasks in their missions), $100$ problem instances are created in a similar way as before.
	Fig.~\ref{fig:compareOthers} shows performance comparison of the different algorithms.
	Iterative Greedy has the best performance among all algorithms, and we observe $6$ to $10\%$ improvement over the baseline Na\"ive Greedy algorithm for small to moderate problem size.  
	We observe that, as the number of robots increases, the difference between the performance of all algorithms start to diminish.  This supports the intuition that as number of robots increases, the human operator is required to distribute their time to more and more robots, thus decreasing their effectiveness.  
	From the plots, we also observe the effectiveness of the Iterative Greedy in improving relative performance when applied in combination with Na\"ive Greedy and Comparison Greedy.  This indicates that the Iterative Greedy technique can be used to further improve any greedily-generated schedule. 

	\begin{figure}[]
		\centering
		\includegraphics[width=\columnwidth]{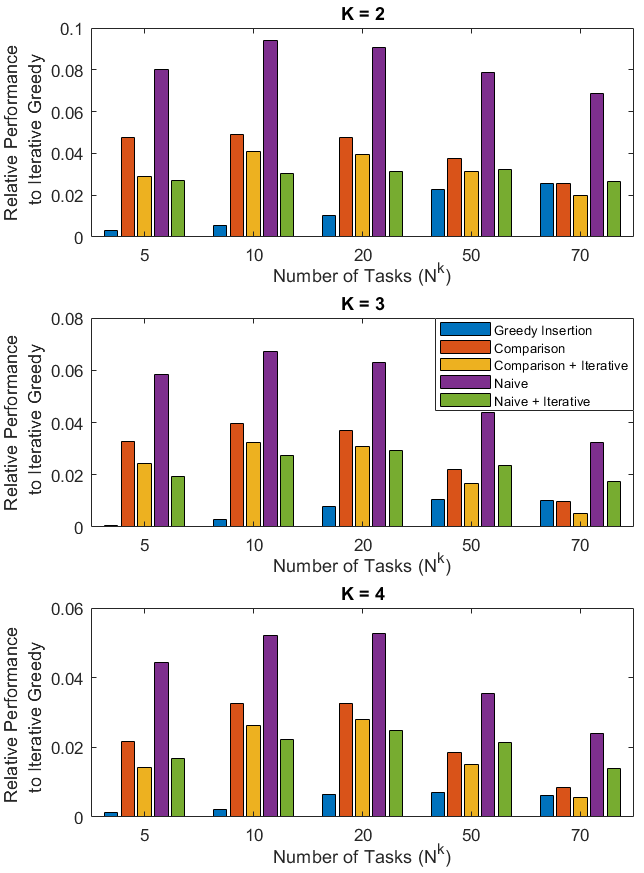}
		\caption {Performance comparison of baseline greedy solution techniques relative to the proposed Iterative Greedy algorithm.  The plots show relative performance of different techniques for up to $4$ robots and $70$ tasks each.}
		\label{fig:compareOthers}
	\end{figure}
	
	\subsection{Example Problem Instance}
	
	\begin{figure}[]
		\centering
		\includegraphics[width=\columnwidth]{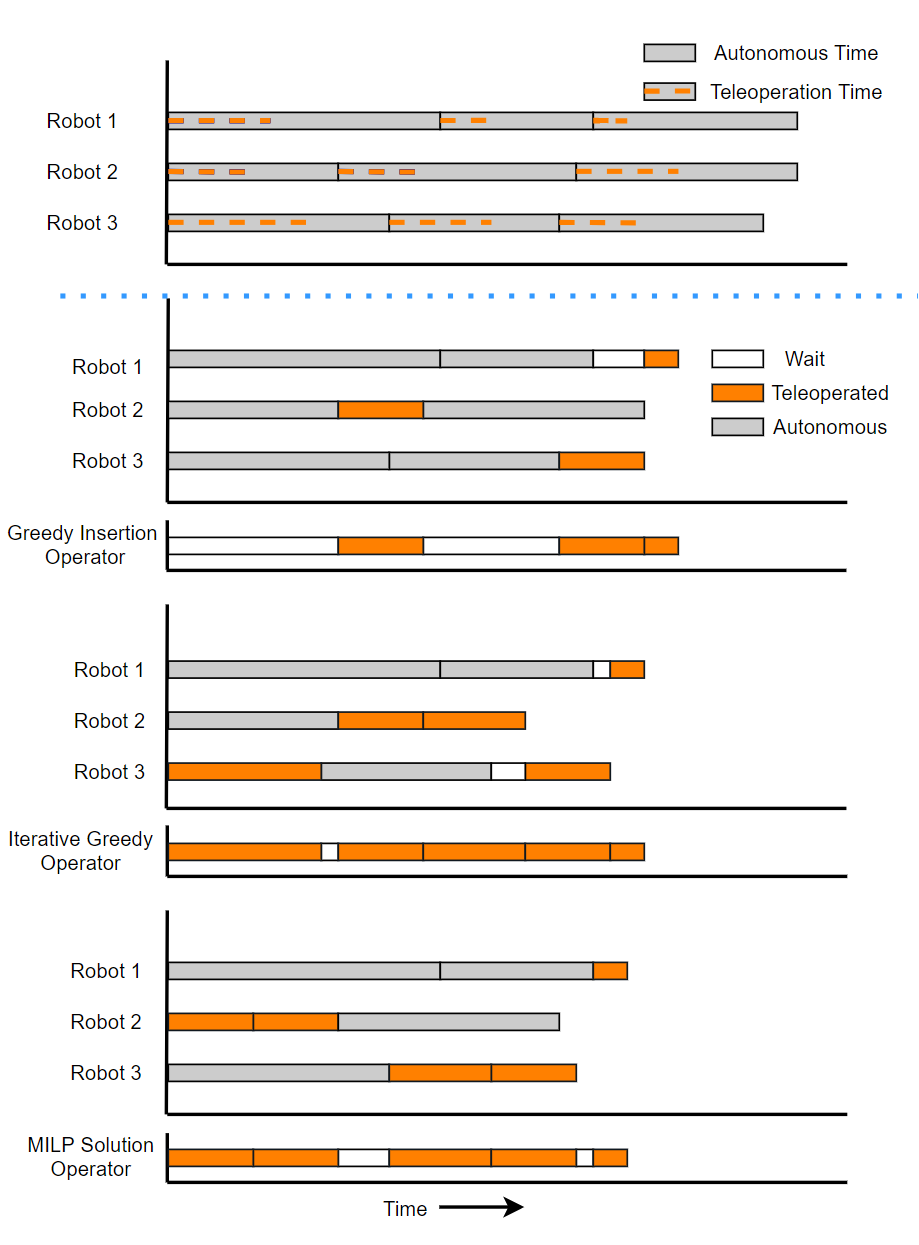}
		\caption {Scheduling of Iterative Greedy, Greedy Insertion and MILP Solution on a Multi-robot Mission Instance.}
		\label{fig:Example}
	\end{figure}
	
	Fig.~\ref{fig:Example} shows an example instance of the scheduling problem with three robots.  First, the Greedy Insertion algorithm generates a schedule that reduces makespan but contains long gaps (idle time) in operator's schedule.  Then the Block Removal algorithm removes these gaps and results in a schedule with very little idle time for the operator.  The MILP solution shows that a better performing schedule is possible even with a greater idle time. 
	
	\section{Conclusions and Discussions}
	
	In this paper, we present a problem of scheduling a human operator to a team of multiple robots, such that the team makespan is minimized.  We show that this problem is NP-Hard and develop the Iterative Greedy algorithm that cycles through two sub-routines: Greedy Insertion and Block Removal.  This algorithm generates a greedy schedule in each iteration, and improves it by removing \textit{blockages} when needed.  The algorithm scales well with problem size and produce smaller makespan than other greedy solution techniques.  It is also shown that the Iterative Greedy algorithm can be applied to any greedily-generated schedule to further improve the performance.  
	For future research, our goal is to further develop the model by allowing imperfect information and possibility of mission re-planning for the robots.  The solution technique will also benefit from the ability to adapt the current schedule online based on new observations.
	
	\bibliographystyle{IEEEtran}

\end{document}